\documentclass{article}

\usepackage[preprint]{neurips_2023}

\usepackage[utf8]{inputenc} 
\usepackage[T1]{fontenc}    
\usepackage{hyperref}       
\usepackage{url}            
\usepackage{booktabs}       
\usepackage{amsfonts}       
\usepackage{nicefrac}       
\usepackage{microtype}      
\usepackage{xcolor}         

\usepackage{amsmath}
\usepackage{amsthm}
\usepackage{amssymb}
\usepackage{mathtools}
\usepackage{mathrsfs}
\mathtoolsset{showonlyrefs}

\usepackage[linesnumbered,ruled,vlined]{algorithm2e}

\usepackage{accents}

\usepackage{bm}

\newtheorem{thm}{Theorem}[]

\usepackage{theoremref} 

\usepackage{graphicx}

\usepackage{siunitx} 
\newcommand{\numc}[1]{\num[group-separator={,}, group-minimum-digits={3}]{#1}}

\newcommand{\ith}{i^\text{th}}
\newcommand{\thetanu}{\theta^\text{nu}}
\newcommand{\thetau}{\theta^\text{u}}
\newcommand{\nun}{{n_{\text{u}}}}
\newcommand{\nnun}{{n_\text{nu}}}
\newcommand{\equalsreason}[1]{\overset{\text{(#1)}}{=}}

\title{Coagent Networks: Generalized and Scaled}

\author{
James E. Kostas\\ University of Massachusetts \And Scott M. Jordan\\ University of Alberta \And Yash Chandak\\ Stanford University \And Georgios Theocharous \thanks{Dr. Theocharous, who sadly passed away before the submission of this paper, made substantial contributions to the research and manuscript.}\\ Adobe Research \And Dhawal Gupta\\ University of Massachusetts \And Martha White\\ University of Alberta \And Bruno Castro da Silva\\ University of Massachusetts \And Philip S. Thomas\\ University of Massachusetts}

\begin{document}

\maketitle

\begin{abstract}
Coagent networks for reinforcement learning (RL) \citep{Thomas2011a} provide a powerful and flexible framework for deriving principled learning rules for arbitrary stochastic neural networks.
%
%
%
The coagent framework offers an alternative to backpropagation-based deep learning (BDL) that overcomes some of backpropagation's main limitations.
For example, coagent networks can compute different parts of the network \emph{asynchronously} (at different rates or at different times), can incorporate non-differentiable components that cannot be used with backpropagation, and can explore at levels higher than their action spaces (that is, they can be designed as hierarchical networks for exploration and/or temporal abstraction).
However, the coagent framework is not just an alternative to BDL; the two approaches can be blended: BDL can be combined with coagent learning rules to create architectures with the advantages of both approaches.
This work generalizes the coagent theory and learning rules provided by previous works; this generalization provides more flexibility for network architecture design within the coagent framework.
This work also studies one of the chief disadvantages of coagent networks: high variance updates for networks that have many coagents and do not use backpropagation.
We show that a coagent algorithm with a policy network that does not use backpropagation can scale to a challenging RL domain with a high-dimensional state and action space (the MuJoCo Ant environment), learning reasonable (although not state-of-the-art) policies.
These contributions motivate and provide a more general theoretical foundation for future work that studies coagent networks.

\end{abstract}

\section{Introduction}
\label{sec:intro}

\emph{Coagent networks} \citep{Thomas2011a} are a class of stochastic neural networks (that is, networks that contain non-deterministic nodes) for reinforcement learning (RL).
These neural networks are composed of stochastic nodes called \emph{coagents} (short for conjugate agents);
each coagent is an RL algorithm.
The coagents cooperatively choose actions.
Therefore, the network as a whole is an RL policy.
See Figure \ref{fig:coagents_basics} for an example coagent network to build intuition.

\begin{figure}[t]
    \includegraphics[width=0.5\linewidth]{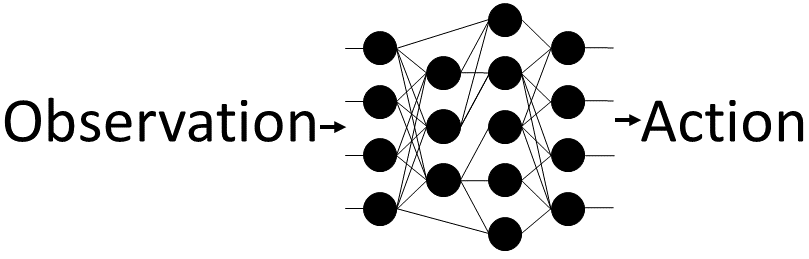} 
    \centering
\caption{A simple coagent network.
Each circle represents a stochastic node called a conjugate agent (coagent).
Each coagent is an RL algorithm that learns and acts cooperatively with the other coagents to choose actions.
As discussed in Sections \ref{sec:intro} and \ref{sec:background}, prior work provides theoretically-grounded learning rules for coagent networks such as this one, including the case where the network is \emph{asynchronous}, that is, the case where network components are not computed simultaneously or at the same rate.
}
\label{fig:coagents_basics}
\end{figure}

Training these networks in a principled way is non-trivial.
We cannot use backpropagation because of the stochastic nodes.
(In special cases, the reparameterization trick may be used, but this is not true in general.)
Prior work provides an answer: Each coagent updates according to a coagent learning rule that causes the coagent network as a whole to follow a principled, theoretically-grounded policy gradient \citep{Thomas2011a}.

Prior work also studies the case where the network is run \emph{asynchronously}, that is, the case where different components of the network execute at different times and/or run at different rates \citep{kostas2020asynchronous}.
Prior work provides theoretically-grounded learning rules for this setting as well.
See Section \ref{sec:background} for a formal introduction to these learning rules.

These prior works provide a powerful and flexible framework for deriving principled learning rules for arbitrary stochastic neural networks.
However, there are a few significant gaps in the literature:
\textbf{1)} Prior theory and learning rules do not address the case where shared parameters are used with asynchronous networks.
Given the prominent role of parameter sharing in modern machine learning (for example, transformers and convolutional networks), this is a major limitation.
\textbf{2)} Prior work does not show that coagent networks can scale well in practice to modern high-dimensional RL benchmark problems.
In this work, we study and address these two gaps; see Section \ref{sec:contributions} below for further discussion.

\subsection{The Coagent Toolbox for Neural Network Design}

This work is basic research that investigates the learning theory and behavior of coagent networks, rather than exploring the advantages the coagent framework offers.
While these advantages are not the focus of this work, they do motivate it, and therefore are briefly outlined below.
For the benefit of readers less familiar with the coagent framework, these advantages are explained in greater detail in supplementary material Section \ref{sec:advantages}.

The coagent framework offers both 1) an alternative to backpropagation-based deep learning (BDL) and 2) a method for combining coagent learning rules with backpropagation, so that the advantages of both approaches can be leveraged for a single network
(see supplementary material Section \ref{sec:combining_backprop} for more details).
For both use cases, the coagent framework provides a principled, theoretically-grounded method for the derivation of learning rules for arbitrary stochastic neural networks for RL.
Specifically, the coagent framework provides neural network designers with a set of powerful, flexible tools that that are unavailable using the BDL approach alone:

\textbf{Non-differentiable components}: Coagent networks can incorporate non-differentiable components that BDL networks cannot use.
This advantage overcomes one of the most restrictive aspects of BDL and opens up a drastically wider range of options for network design.
See supplementary material Section \ref{sec:non-differentiable} for further discussion and examples.

\textbf{Asynchronous networks for temporal abstraction:} Coagent networks can run \emph{asynchronously}, with different parts of a network executing at different times and/or running at different rates, while backpropagation-based networks must be computed synchronously \citep{kostas2020asynchronous}.
This advantage can be used to 1) design networks such that different components run at different rates and/or 2) design network components such as termination functions \citep{sutton1999between}.
In this way, one can use the asynchronous coagent framework as a powerful tool for the design of \emph{arbitrary} option-critic-like \citep{bacon2017option} architectures for temporal abstraction.
See supplementary material Section \ref{sec:temporalAbstraction} for more discussion.

\textbf{Hierarchical architectures for exploration above the level of primitive actions:} Similarly, coagent networks (asynchronous or not) allow for exploration at higher levels than the output layer: since networks have internal nodes that are themselves stochastic RL algorithms, internal components of the network that do not directly affect the action can explore, allowing for hierarchical architectures that facilitate exploration.
See supplementary material Section \ref{sec:hierarchical} for further discussion.

\textbf{Asynchronous networks for distributed computation:} 
Asynchronous coagent networks can also be designed to be distributed and run in parallel.
For example, large networks could be distributed over a CPU cluster and computed asynchronously in a way that cannot be done with BDL (no GPUs needed), or coagent networks that use backpropagation could be run asynchronously across multiple GPUs, without the need for the GPUs to continuously synchronize (as is required for BDL).
See supplementary material Section \ref{sec:distributing} for more discussion.

\subsection{Contributions}
\label{sec:contributions}

This work is basic research that investigates the two previously discussed gaps in prior work.
Specifically, we study the gap in the theory where shared parameters are used with asynchronous networks, and we study the question of whether coagent networks can scale well in practice to a modern high-dimensional RL benchmark problem.
There are two primary contributions:

\textbf{1) A generalized learning rule:} In this work, we generalize the theory of coagent networks, unifying the results of \citet{kostas2020asynchronous} and \citet{zini2020coagent}.
Specifically, we generalize the theory and learning rules to cover the case where \emph{asynchronous} coagent networks use \emph{shared parameters}, where a single learnable parameter may be used in more than one ``location'' in a neural network architecture.
Parameter sharing is a well-established tool for architecture design in many fields of machine learning;
well-known uses of parameter sharing include convolutional filters for computer vision \citep{NIPS1989_53c3bce6} and natural language processing, and parameter sharing for recurrent networks and transformers \citep{attention}.
Parameter sharing has also been used extensively for RL specifically.
For example, BDL architectures for RL pixel-input problems often use convolutional neural networks \citep{mnih2015human}, and the well-known decision transformer \citep{decision_transformer} utilizes the transformer architecture.
Thus, this work's generalized theory eliminates a major limitation of the coagent framework.
This generalization gives users even more freedom to design arbitrary stochastic architectures that have principled, theoretically-grounded learning rules.

\textbf{2) Large backpropagation-free coagent policy networks can scale to a challenging problem:}
Despite the advantages of coagent networks, there are legitimate concerns about the practicality of these networks.
Prior work does not show that coagent networks can successfully learn high-dimensional modern benchmark RL domains (for example, the MuJoCo environments) or even scale to the large policy network sizes needed for such problems.
Worse, there is good reason to think that these networks cannot scale practically due to the structural credit assignment problem:
The foundational coagent learning rule is unbiased but can give high variance gradient estimates, and this variance becomes larger as the number of coagents in the network increases.
While this is not a serious concern when using small coagent networks or large coagents networks with few coagents (one straightforward approach to designing the latter class is to combine coagent learning rules with backpropagation), the variance of the gradient updates is a more serious concern when using large coagent networks with many coagents and no backpropagation.

This work addresses these concerns by showing that even a large coagent policy network that does not use backpropagation can, in fact, successfully scale to a challenging high-dimensional problem: the MuJoCo Ant environment.
This is a worst-case design choice with respect to the concern about structural credit assignment, since this choice increases the number of coagents, thus increasing the variance of the gradient estimate.
This result demonstrates that coagent networks are scalable, even in this ``worst-case'' many-coagent-no-backpropagation case.
Our hope is that these results motivate further research of coagent networks and their uses for neural network design.

\section{Related Work}
\label{sec:related_work}

Backpropagation-based deep learning (BDL) networks are behind many of the recent exciting breakthroughs in RL for simulated robotics problems or game-playing.
For example, \citet{mnih2013playing} and \citet{mnih2015human} introduced the deep Q-learning algorithm (DQN) and showed that their algorithm and BDL architecture can learn to play Atari games using pixel-only observations.
\citet{schulman2017proximal} and \citet{haarnoja2018soft} proposed the proximal policy optimization (PPO) and soft actor-critic (SAC) algorithms, respectively; these algorithms are two of the most popular RL algorithms at this time and are both designed specifically for use with BDL.
Both papers used MuJoCo domains to demonstrate the effectiveness of their algorithms.
\citet{silver2017mastering} showed that a BDL-based RL (DRL) algorithm can learn to beat a world champion in the game of Go (a game that, until recently, was considered to be largely unsolved by the artificial intelligence community).
\citet{wurman2022outracing} showed that a DRL algorithm can learn to outperform a human champion in the game of Gran Turismo.
Our work is inspired both by the successes of DRL and its limitations; coagent networks offer a connectionist alternative to DRL that may overcome some of DRL's limitations.
Additionally, by including BDL networks as components of the coagent networks, coagent networks may add new capabilities to existing DRL algorithms and techniques.

\citet{barto1985learning}, \citet{williams1992simple}, and \citet{narendra1989learning} proposed learning rules for stochastic neural networks; the motivations and approaches of these works are closely related to the motivations and approaches of coagent learning rules.
Building on these ideas, \citet{Thomas2011a} proposed coagent networks for RL, and \citet{Thomas2011b} developed the framework further by studying coagent learning rules that are based on policy gradient algorithms.

Several more recent works also study coagent networks.
\citet{kostas2020asynchronous} introduced asynchronous coagent networks.
\citet{zini2020coagent} generalized previous coagent theory and gave a learning rule for the case where network parameters are shared.
They also introduced and studied a type of coagent network called the Feedforward Option Network.
\citet{gupta2021structural} studied the structural credit assignment problem for coagent learning rules and coagent networks in the supervised learning setting.
They analyzed a variety of coagent learning rules, critics, coagent architectures, and many other design choices.
\citet{chung2021} proposed the MAP propagation algorithm.
This algorithm is closely related to coagent algorithms; MAP propagation reduces the variance of coagent updates at the cost of computational efficiency and bias.

\section{Background}
\label{sec:background}

We study a Markov decision process (MDP)
$M=(\mathcal S, \mathcal A, \mathcal R, P, R, d_0,\gamma)$, where $\mathcal S$ is the state space, $\mathcal A$ is the action space, $\mathcal R$ is the space of possible rewards, $P$ is the transition function, $R$ is the reward function, $d_0$ is the initial state distribution, and $\gamma \in [0,1]$ is the reward discount parameter.
Let $t \in \{0, 1, 2, \dotsc\}$ represent the time step.
Let $S_t$, $A_t$, and $R_t$ be random variables that represent the state, action, and reward respectively at time $t$ and take values in $\mathcal S$, $\mathcal A,$ and $\mathcal R$ respectively.
The transition function, $P{:} \mathcal S {\times} \mathcal A {\times} \mathcal S {\to} [0, 1]$, is defined as $P(s,a,s')\coloneqq\Pr(S_{t+1}{=}s'|S_t{=}s,A_t{=}a)$.
The reward function $R{:} \mathcal S {\times} \mathcal A {\times} \mathcal S {\times} \mathcal R {\to} [0, 1]$, is defined as $R(s,a,s',r)\coloneqq\Pr(R_t{=}r|S_t{=}s,A_t{=}a,S_{t+1}{=}s')$.
The initial state distribution, $d_0{:}\mathcal S \to [0, 1]$, is defined as $d_0(s)\coloneqq\Pr(S_0{=}s)$.
Let $\Theta$ be the feasible set.
A policy, $\pi:\mathcal S \times \mathcal A \times \Theta \to [0,1]$ is a method of selecting actions given the state, defined as $\pi(s,a, \theta)\coloneqq\Pr(A_t{=}a|S_t{=}s)$, where the distribution is parameterized by $\theta$. 
We assume that all $\theta \in \Theta$ are vectors of real-valued numbers of length $n$; that is, $\Theta \subseteq \mathbb R^n$.
%
%
The discounted return at and after timestep $t$, $G_t$, is $G_t \coloneqq \sum_{k=0}^\infty\gamma^{k} R_{t+k}$. 
The typical goal of an RL algorithm is to maximize the objective function $J:\Theta \to \mathbb R$, which is defined as $J(\theta)\coloneqq\mathbf{E}\left [\sum_{t=0}^\infty \gamma^t R_t\right].$ 

\emph{Policy gradient algorithms} are a popular class of RL algorithms which aim to maximize the objective function by performing gradient ascent on the objective using estimates of $\nabla J(\theta)$.

\subsection{Coagent Networks}

\emph{Coagent networks} are stochastic feedforward networks for RL, consisting of a network of conjugate agents, or \emph{coagents}, each of which executes an RL algorithm based on its inputs and outputs; together, these coagents form a coagent network.
\citet{Thomas2011a} give principled learning rules for these networks.
Let $m$ to be the number of coagents in the network.
Since a coagent network is feedforward, we enumerate the coagents that comprise the network in order of computation: coagent $1$, coagent $2, \dotsc,$ coagent $m$.

Given the coagent network parameters $\theta \in \Theta$, we write the $\ith$ coagent's parameters as $\tilde \theta_i \in \tilde \Theta_i$, where $\tilde \Theta_i \subseteq \mathbb R^{n_i}$ is the corresponding feasible set and $n_i$ is the number of parameters this coagent has.
We use the tilde above the $\tilde \theta_i$ and $\tilde \Theta_i$ symbols to denote that we are referring to the vector that constitutes the $\ith$ coagent's parameters (for $i \in \{1, 2, \dotsc, m\}$), as opposed to the $\ith$ parameter in $\theta$, which is a real number that we denote without a tilde: $\theta_i \in \mathbb R$, for $i \in \{1, 2, \dotsc, n\}$.
So $\theta = [\tilde \theta_1^\intercal, \tilde \theta_2^\intercal, \dotsc, \tilde \theta_m^\intercal]^\intercal = [\theta_1, \theta_2, \dotsc, \theta_n]^\intercal$.

Let $\mathcal U_i$ be the output space of the $\ith$ coagent; this can be thought of as the coagent's ``action space''.
Let $U^i_t \in \mathcal U_i$ be the random variable that represents the $\ith$ coagent's output at timestep $t$.

The input space of the $\ith$ coagent, $\mathcal X_i$, consists of the state space of $M$ and the outputs of the previous coagents; define this space as $\mathcal X_i \coloneqq \mathcal S \times \mathcal U_1 \times \mathcal U_2 \times \dotsc \times \mathcal U_{i-1}$.
Let $X^i_t \in \mathcal X_i$ be the random variable that represents the $\ith$ coagent's input at timestep $t$.
Note that the $\ith$ coagent need \emph{not} compute a function that is affected by the entire variable $X^i_t$: if network structure dictates that a coagent is not directly affected by some or all of the state space or some or all of the previous coagents' outputs, this theory still encapsulates the network structure; the function $\pi_i$ simply ``ignores'' parts of $\mathcal X_i$ as required.

The policy of the $\ith$ coagent, $\pi_i:\mathcal X_i \times \mathcal U_i \times \tilde \Theta_i \to [0,1]$ is a method of selecting actions given the state, defined as $\pi_i(x,u, \tilde \theta_i)\coloneqq\Pr(U_t{=}u|X_t{=}x)$.

\subsection{Asynchronous Continuous-Time Coagent Networks}

\citet{kostas2020asynchronous} study the reinforcement learning setting in which each time step in the MDP $M$ is broken up into $n_\text{atomic} \in \mathbb Z^+$ \emph{atomic time steps}, where $\mathbb Z^+$ denotes the positive integers.
This formulation allows for arbitrarily-precise approximations of continuous-time \emph{asynchronous} coagent networks.
In other words, different parts of the networks may be run at different rates to facilitate temporal abstraction \citep{bacon2017option}; coagent networks with cycles (recurrent networks) may be designed; and/or coagent networks may be distributed across CPU clusters or the internet and run asynchronously.
In this setting, at each atomic time step, each coagent either updates its output or does not; we call this process \emph{execution}.
That is, if the $\ith$ coagent executes at a given atomic time, it computes some output from its input using its policy $\pi_i$.
If a coagent does not execute at a given atomic time step, its output remains the same as at the previous atomic time step.

In this setting, each coagent has an \emph{execution function}, which defines the probability of execution given the coagent's inputs.
In this work, we do not formalize the execution function (see the work of \citet{kostas2020asynchronous} for a more formal treatment), since it is not necessary for our results, and instead simply let the random variable $E^i_t \in \{0, 1\}$ represent whether the $\ith$ coagent executes at atomic time step $t$; $E^i_t$ takes a value of $1$ if the coagent executes and $0$ otherwise.

The \emph{asynchronous local policy gradient} for coagent $i$ is defined as:

\begin{equation}
    \Delta_i(\tilde \theta_i) \coloneqq
    \mathbf{E} \left[ \sum_{t=0}^\infty E^i_t \gamma^t G_t 
    %
    \frac{\partial \ln \big( \pi_i \big ( X_t, U_t, \tilde \theta_i \big ) \big )}{\partial \theta_i} \right],
\end{equation}
where $t$ refers to atomic time steps.

\citet{kostas2020asynchronous} prove that the policy gradient of an asynchronous coagent network is composed of the coagents' asynchronous local policy gradients:

\begin{equation}
\label{thm:acpgt}
    \nabla J(\theta)=\left [\Delta_1(\tilde \theta_1)^\intercal , \Delta_2(\tilde \theta_2)^\intercal,\dotsc,\Delta_m(\tilde \theta_m)^\intercal\right].
\end{equation}

Notice that this theorem provides us with the $n$ real-valued components of the policy gradient:

\begin{align}
\label{thm:acpgt_individual}
    \nabla J(\theta)=&\left [\Delta_1(\tilde \theta_1)^\intercal , \Delta_2(\tilde \theta_2)^\intercal,\dotsc,\Delta_m(\tilde \theta_m)^\intercal\right ]\\
    =& \left[\frac{\partial J(\theta)}{\partial \theta_1}, \frac{\partial J(\theta)}{\partial \theta_2}, \dotsc, \frac{\partial J(\theta)}{\partial \theta_n}\right].
\end{align}

In this paper, we use the convention that vectors are column vectors by default, but that vectors with comma-separated values are row vectors, so all three vectors in the equations above are $1 \times n$ row vectors (since they are the Jacobian of $J$).

\section{The Shared Parameter Asynchronous Coagent Policy Gradient Theorem}

In this section, we give a generalized coagent policy gradient theorem that covers both \textbf{1)} the case of (continuous-time) asynchronous and recurrent networks and \textbf{2)} the case of networks with shared parameters.
\citet{kostas2020asynchronous} give the asynchronous proof but do not cover the case of shared parameters, and \citet{zini2020coagent} give the shared parameter proof but do not cover the case of asynchronous networks.
Note that \citet{zini2020coagent} do discuss how to extend their result to the asynchronous setting (see Remark 12 and Appendix B.1.1).
However, they do not provide a formal proof, and it is not immediately clear that the route they suggest (applying the logic of \citet{kostas2020asynchronous} to their synchronous theorem) would be correct and straightforward, since the two papers use significantly different formulations of the coagent setting.
Therefore, in this section, we formally prove the result for the asynchronous shared parameter setting, which unifies the results of \citet{kostas2020asynchronous} and \citet{zini2020coagent}.

We define the number \emph{unique parameters} to be the number of network parameters, not counting shared ``duplicate'' parameters.
We define the number of \emph{non-unique parameters} to be the number of network parameters, counting shared ``duplicates''.
For example, if a simple convolutional network consisted of a single $3 \times 3$ nine-parameter filter that is applied to the input $64$ times, this network would have nine unique parameters and $9 (64) = 576$ non-unique parameters.

Suppose we wish to analyze the setting where an asynchronous coagent network is designed to share parameters, and has $\nun$ unique parameters, and $\nnun$ non-unique parameters, for $\nun \leq \nnun$.
Denote the vector of unique parameters as $\thetau \in \mathbb R^\nun$ and the vector of non-unique parameters as $\thetanu \in \mathbb R^\nnun$.
Notice that for any given architecture, there exists a function that, given the unique parameters, outputs the non-unique parameters.
We denote this function as $f:\mathbb R^\nun \to \mathbb R^\nnun$ such that $f(\thetau) \coloneqq \thetanu$.
Similarly, there exists a function that takes the index of a unique parameter in $\{1, 2, \dotsc, \nun\}$ and produces the set of corresponding non-unique indices in $\mathcal P(\{1, 2, \dotsc, \nnun\})$ (where $\mathcal P$ denotes the powerset);
we denote this function as $g:\{1, 2, \dotsc, \nun\} \to \mathcal P(\{1, 2, \dotsc, \nnun\})$.

The shared parameter asynchronous coagent policy gradient theorem (SPAT) states that the policy gradient with respect to the unique parameters consists of $\nun$ components, such that each corresponding component is the sum of the gradients of the corresponding component(s) of the non-unique policy gradient given by \eqref{thm:acpgt_individual}.

For example, consider the simple case where there are three non-unique parameters, two unique parameters, and the second and third non-unique parameters are shared.
That is, $\thetanu = [\thetanu_1, \thetanu_2, \thetanu_3]^\intercal$, $\thetau = [\thetau_1, \thetau_2]^\intercal$, $\thetau_1 = \thetanu_1$, and $\thetau_2 = \thetanu_2 = \thetanu_3$.
We know from \eqref{thm:acpgt_individual} that, if no parameters were shared, the policy gradient $\nabla J(\thetanu)$ would be a vector with three components given by \eqref{thm:acpgt_individual}: $\nabla J(\thetanu) = [\frac{\partial J(\theta)}{\partial \thetanu_1}, \frac{\partial J(\theta)}{\partial \thetanu_2}, \frac{\partial J(\theta)}{\partial \thetanu_3}]$.
The SPAT states that the two-component vector that is the policy gradient for the unique parameters is simply the sum of the gradients of the corresponding component(s) in the non-unique policy gradient: $\frac{\partial J(f(\thetau))}{\partial \thetau} = [\frac{\partial J(\theta)}{\partial \thetanu_1}, \frac{\partial J(\theta)}{\partial \thetanu_2} + \frac{\partial J(\theta)}{\partial \thetanu_3}]$.

\begin{thm}[Shared Parameters Asynchronous Coagent Policy Gradient Theorem] \label{thm:spat} 
$$\frac{\partial J(f(\thetau))}{\partial \thetau} = \left[\sum_{i \in g(1)}\frac{\partial J(f(\thetau))}{\partial \thetanu_i}, \sum_{i \in g(2)}\frac{\partial J(f(\thetau))}{\partial \thetanu_i}, \dotsc, \sum_{i \in g(\nun)}\frac{\partial J(f(\thetau))}{\partial \thetanu_i}\right].$$
\end{thm}

\begin{proof}
\begin{align*}
    \frac{\partial J(f(\thetau))}{\partial \thetau} \equalsreason{a}& \frac{\partial J(f(\thetau))}{\partial f(\thetau)} \frac{\partial f(\thetau)}{\partial \thetau} \\
    \equalsreason{b}& \frac{\partial J(\thetanu)}{\partial \thetanu} \frac{\partial f(\thetau)}{\partial \thetau},
\end{align*}
where \textbf{(a)} follows from the chain rule, and \textbf{(b)} follows from the fact that $f(\thetau) \coloneqq \thetanu.$

Notice that $\frac{\partial J(\thetanu)}{\partial \thetanu}$ is a $1 \times \nnun$ row vector given by \eqref{thm:acpgt_individual}, and $\frac{\partial f(\thetau)}{\partial \thetau}$ is a $\nnun \times \nun$ one-hot matrix, each entry of which in $\{0, 1\}$ indicates whether the unique index (column index) corresponds to the non-unique index (row index).
Next, we define these one-hot entries: for all $x \in \{1, 2, \dotsc, \nun\}$, $y \in \{1, 2, \dotsc, \nnun\}$, let $H_{xy} \in \{0, 1\}$ be the indicator variable that indicates whether the unique parameter with index $x$ corresponds with the non-unique parameter with index $y$.
By inspection, 

\begin{align*}
    \frac{\partial J(f(\thetau))}{\partial \thetau} =&
    \left[\frac{\partial J(f(\thetau))}{\partial \thetanu_1}, \frac{\partial J(f(\thetau))}{\partial \thetanu_2}, \dotsc, \frac{\partial J(f(\thetau))}{\partial \thetanu_n}\right] \frac{\partial f(\thetau)}{\partial \thetau}\\
    =& \left[\frac{\partial J(f(\thetau))}{\partial \thetanu_1}, \frac{\partial J(f(\thetau))}{\partial \thetanu_2}, \dotsc, \frac{\partial J(f(\thetau))}{\partial \thetanu_n}\right] 
    \begin{bmatrix}
    H_{11} & H_{21} & \cdots & H_{\nun1}\\
    H_{12} & H_{22} & & \vdots \\
    \vdots & & \ddots \\
    H_{1\nnun} & \cdots & & H_{\nun\nnun}
\end{bmatrix}\\
    =& \left[\sum_{i \in g(1)}\frac{\partial J(f(\thetau))}{\partial \thetanu_i}, \sum_{i \in g(2)}\frac{\partial J(f(\thetau))}{\partial \thetanu_i}, \dotsc, \sum_{i \in g(\nun)}\frac{\partial J(f(\thetau))}{\partial \thetanu_i}\right].
\end{align*}
See the illustrative example below for more intuition regarding this final step.
\end{proof}

We provide an example to make the final step of the proof more intuitive, using the example two-unique-parameter/three-non-unique-parameter coagent network above:

\begin{align*}
    \frac{\partial J(f(\thetau))}{\partial \thetau} =&
    \left[\frac{\partial J(f(\thetau))}{\partial \thetanu_1}, \frac{\partial J(f(\thetau))}{\partial \thetanu_2}, \frac{\partial J(f(\thetau))}{\partial \thetanu_3}\right] \frac{\partial f(\thetau)}{\partial \thetau}\\
    =& \left[\frac{\partial J(f(\thetau))}{\partial \thetanu_1}, \frac{\partial J(f(\thetau))}{\partial \thetanu_2}, \frac{\partial J(f(\thetau))}{\partial \thetanu_3}\right] \begin{bmatrix}
    1 & 0\\
    0 & 1 \\
    0 & 1 \\
\end{bmatrix}\\
=& \left[\frac{\partial J(f(\thetau))}{\partial \thetanu_1}, \frac{\partial J(f(\thetau))}{\partial \thetanu_2} + \frac{\partial J(f(\thetau))}{\partial \thetanu_3}\right].
\end{align*}

\subsection{Example Update Rule}

As a simple example of a learning rule based on Theorem \ref{thm:spat}, consider the REINFORCE \citep{williams1992simple} analogue for the asynchronous shared weights setting.
We first give detailed step-by-step instructions to build intuition, and then provide a succinct summary in the form of pseudocode.

First, after each episode, compute the asynchronous update to each coagent's parameters, as described by \citet{kostas2020asynchronous} (we will denote this gradient estimate for coagent $i$ as $\widetilde \nabla_i$):
\begin{align*}
\widetilde \nabla_i = \!\!\!\!\!\! \sum_{t \in \{t^i_1, t^i_2, \dotsc\}} \!\!\!\!\!\!\!\! \gamma^t G_t \left(\frac{\partial \ln\left ( \pi_i\left ( X^i_t, U^i_t, \tilde \theta_i \right ) \right )}{\partial \tilde \theta_i}\right),
\end{align*}
where the elements of $\{t^i_1, t^i_2, \dotsc\}$ are the times that the $\ith$ coagent executed during this episode.
Notice that this vector has the length of the number of \emph{non-unique} parameters of coagent $i$.
Next, update as normal, but, for any shared parameters, simply \emph{sum} the contributions of each relevant $\widetilde \nabla_i$.
Stating this last step more formally:

Compute the non-unique gradient vector (of length $\nnun$):
\begin{align*}
    \widetilde \nabla^{\text{nu}} \gets [\widetilde \nabla_1^\intercal, \widetilde \nabla_2^\intercal, \cdots, \widetilde \nabla_m^\intercal],
\end{align*}
and then use that to compute the unique gradient vector (of length $\nun$):
\begin{align*}
    \widetilde \nabla^{\text{u}} \gets \left[\sum_{i \in g(1)} \widetilde{ \nabla}^{\text{nu}}_i, \sum_{i \in g(2)} \widetilde{ \nabla}^{\text{nu}}_i, \dotsc, \sum_{i \in g(\nun)} \widetilde{ \nabla}^{\text{nu}}_i \right],
\end{align*}
where $ \widetilde{ \nabla}^{\text{nu}}_i$ is the $\ith$ component of $\widetilde \nabla^{\text{nu}}$.
Then proceed with the gradient update:
\begin{align*}
    \thetau \gets \thetau + \alpha \widetilde \nabla^{\text{u}},
\end{align*}
where $\alpha$ is the learning rate.

This learning rule can be concisely summarized by the following pseudocode (all terms defined above):

\begin{algorithm}

    \For{$i = \{1, 2, \dotsc, m\}$}{
        $\widetilde \nabla_i = \sum_{t \in \{t_1, t_2, \dotsc\}} \gamma^t G_t \left(\frac{\partial \ln\left ( \pi_i\left ( X^i_t, U^i_t, \tilde \theta_i \right ) \right )}{\partial \tilde \theta_i}\right);$ \# Calculate the gradient for each coagent}

    $\widetilde \nabla^{\text{nu}} \gets [\widetilde \nabla_1^\intercal, \widetilde \nabla_2^\intercal, \cdots, \widetilde \nabla_m^\intercal]$; \# Combine these terms to calculate what the gradient would be if no parameters were shared

    $\widetilde \nabla^{\text{u}} \gets \left[\sum_{i \in g(1)} \widetilde{ \nabla}^{\text{nu}}_i, \sum_{i \in g(2)} \widetilde{ \nabla}^{\text{nu}}_i, \dotsc, \sum_{i \in g(\nun)} \widetilde{ \nabla}^{\text{nu}}_i \right]$; \# Account for parameter sharing

    $\thetau \gets \thetau + \alpha \widetilde \nabla^{\text{u}}$; \# Update network parameters
    \caption{Pseudocode for REINFORCE Analogue}
    \label{alg:reinforce}
\end{algorithm}

\section{Experiments}

Previous work does not show that coagent networks can scale to large networks or to challenging high-dimensional domains such as MuJoCo environments.
As discussed in Section \ref{sec:intro}, this is particularly concerning, since while coagent learning rules will give the correct gradient in expectation, the variance of the gradient will increase as the number of coagents increases.

In this section, we show that even a backpropagation-free actor coagent network with many coagents can in fact scale and achieve reasonable results for the MuJoCo Ant v2 environment \citep{todorov2012mujoco}.
The Ant environment has an 111-dimensional state space and an 8-dimensional action space, and requires the RL agent to learn to walk a quadrupedal ``ant'' forward.

The coagent network's learning rule is based on a simple actor-critic algorithm.
The actor network is a two-layer network comprised of $40$ coagents.
Each coagent is a linear function approximator with softmax action selection.
The first layer is comprised of $32$ 224-parameter coagents, each of which outputs a binary signal to the second layer.
The second layer is comprised of eight 363-parameter coagents corresponding to the eight-dimensional action space of the environment.

For the state-of-the-art (SOTA) returns that we provide for comparison, we trained $30$ BDL agents trained with proximal policy optimization (PPO).
We used an implementation based on the Clean-RL library \citep{cleanrl}.

For statistical significance, we use 100 trials (that is, ``runs'' or ``random seeds'') for the coagent network, and 30 trials for the SOTA agents.
For more algorithm and experimental details, see supplementary material Section \ref{sec:experiment_details}.

\subsection{Aims}

\begin{figure}
    \includegraphics[width=0.7\linewidth]{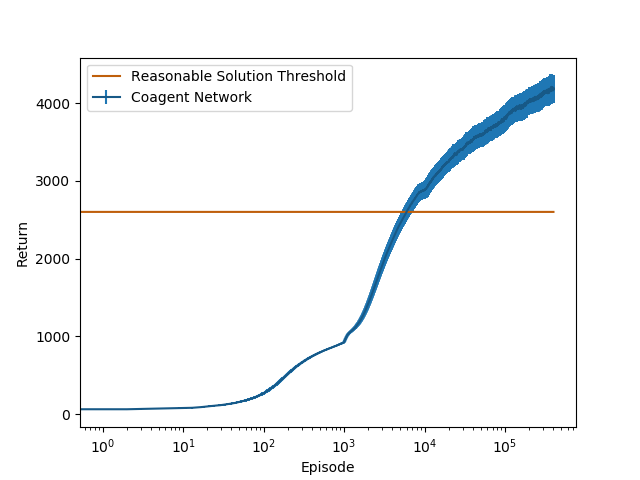}
    \centering
\caption{
This plot displays the coagent networks' learning curve.
To make the plot less noisy, the learning curve displays the moving average over 1000 episodes.
For statistical significance, this curve was generated from 100 trials (that is, 100 runs or seeds).
Error bars indicate standard error.
(The standard error bars are small, since these results have a high degree of statistical significance, due to the number of trials run.)
}
\label{fig:results}
\end{figure}

Note that our goal is \emph{not} to study whether coagents networks’ performance and sample efficiency could match those of SOTA DRL algorithms:
BDL and DRL research and engineering efforts involving a large community of researchers and engineers have been ongoing for years for DRL and decades for BDL, so it would be deeply surprising if coagent networks achieved comparable performance and sample efficiency (since, to the best of our knowledge, this is the first attempt to scale them up to RL problems like high-dimensional MuJoCo domains).
There are many engineering tricks and tweaks used in DRL implementations (see, for example, the more than 50 engineering ``choices'' that \citet{andrychowicz2021matters} study for PPO, many of which are present in the PPO implementation that we use) while our coagent algorithm is a relatively elegant and simple one.
One important example is experience replay (for many popular off-policy DRL algorithms) and/or training multiple times on same data (which the PPO algorithm, which we use for the SOTA data, does), while our algorithm uses a simple actor-critic algorithm that uses each data point only once.

Using the typical DRL engineering techniques would be counterproductive in this work, since our goal is to conduct basic research about the scaling of coagent networks, not to achieve SOTA sample-efficiency and performance, and adding dozens of confounding engineering techniques to the algorithm would defeat this purpose.
However, in future more applied work, techniques like these may result in substantial improvements to sample efficiency and asymptotic performance for practical applications of coagent networks.
See supplementary material Section \ref{sec:experiment_details} for algorithm details; these details demonstrate the relative simplicity and elegance of our algorithm compared to typical DRL implementations.
However, while our goal was not to study whether coagent networks could learn SOTA policies, we do include SOTA results in Table \ref{table:results} for completeness.

Based on visualizations of the environment, policies which give average returns of roughly $\numc{2600}$ (with reasonably low variance) typically indicate a reasonable policy that is able to walk the Ant forward with a reasonable gait at a reasonable speed.
Below, we refer to this return threshold as the \emph{reasonable solution threshold} (RST).
For visualizations which build intuition for this choice of a threshold, see \url{https://sites.google.com/view/coagent-videos/home}.
Our aim was to study whether a backpropagation-free actor with many coagents could learn policies with mean returns of at least the RST.

\begin{table}
  \centering
  \begin{tabular}{|c|c|c|c|}
    \hline
    & RST & Coagent Network & SOTA \\
    \hline
    Average Return & 2600 & 4187 & 5742 \\
    \hline
    Standard Error (Standard Deviation) & - & 163 (1630) & 320 (1755) \\
    \hline
  \end{tabular}
  \caption{The mean, standard error, and standard deviation of the returns.
  The coagent data is generated from 100 trials; the mean and standard deviation are based on the last 1000 episodes of each trial.
  The state of the art (SOTA) return is based on 30 trials of a DRL PPO agent; after training, each policy was run for 1000 episodes each to generate this data.
  }
  \label{table:results}
\end{table}

\subsection{Results}

%

Our results are shown in Table \ref{table:results} and Figure \ref{fig:results}; the coagent networks learn policies with an average return of nearly \numc{4200}, which corresponds to effective and efficient solutions to this challenging problem well above the RST.
These results demonstrate that a backpropagation-free actor coagent network can scale effectively to this problem.

\section{Conclusion}

In this work, we generalized and unified the theory of previous coagent learning rules.
We also demonstrated that even coagent networks with backpropagation-free actors can scale effectively to a challenging, high-dimensional environment.

\bibliographystyle{unsrtnat}
\bibliography{references}

\newpage
\appendix

\section{Uses of the Coagent Framework For Neural Network Design}
\label{sec:advantages}

This work focuses on conducting fundamental research to investigate the learning theory and behavior of coagent networks.
Because the benefits of the coagent framework are the motivation for this work, rather than its contribution, we do not focus on these benefits in the main body of the work.
These benefits are explained in greater detail in this section for the benefit of readers less familiar with the coagent framework and its uses and potential advantages.

\subsection{Asynchronous Coagent Networks for Temporal Abstraction}
\label{sec:temporalAbstraction}

Asynchronous coagent networks can be designed for the purposes of temporal abstraction.
By designing networks so that different components run at different rates, and/or designing components such as termination functions \citep{sutton1999between}, one can use the (asynchronous) coagent framework as a powerful tool for the design of \emph{arbitrary} option-critic-like \citep{bacon2017option} architectures for temporal abstraction.

Consider Figure \ref{fig:async_net}, which showcases the enormous flexibility of the coagent framework's ``asynchronous'' aspect for network design.
The probability of a node (coagent) computing at a given instant (``executing'') may be designed to be deterministic or stochastic, and may even depend on the state space or the outputs of the other coagents (or even that coagent's own output).
This means that this framework generalizes the options framework \citep{sutton1999between} and facilitates the design of arbitrary networks like the option-critic \citep{bacon2017option}.

\begin{figure}[h]
    \includegraphics[width=0.7\linewidth]{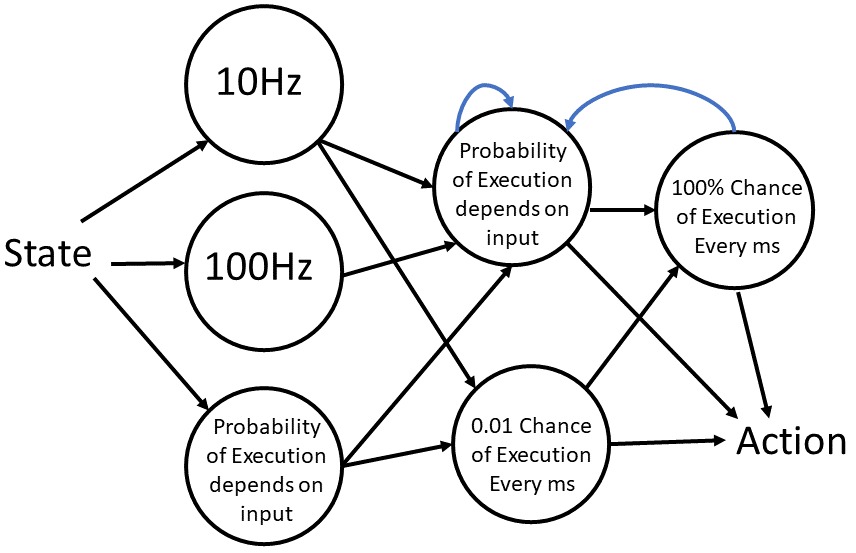}
    \centering
\caption{An illustrative example to show the flexibility asynchronous network design.
Blue connections indicate recurrent (that is, cyclic) connections.
}
\label{fig:async_net}
\end{figure}

\citet{kostas2020asynchronous} focus primarily on the temporal abstraction use-case; see their work for further discussion.

\subsection{Distributing Asynchronous Coagent Networks on Hardware}
\label{sec:distributing}

The asynchronous formulation \citep{kostas2020asynchronous} may be used to distribute and run coagent networks in parallel.
A few interesting potential uses:
\begin{itemize}
    \item Large networks could be distributed over a CPU cluster and computed asynchronously in a way that cannot be done with BDL.
    \item Neural networks that use backpropagation could be run asynchronously across multiple GPUs, without the need for the GPUs to synchronize on every timestep or forward/backward pass (as is required for standard BDL).
    \item Neural networks could be distributed over computer networks or over the internet (with different networks components on different servers) and run asynchronously.
\end{itemize}

Note that, while some BDL architectures may be partially parallelized, these optimizations are specific to the architecture.
Specifically, if component B of a BDL network takes the output of component A, then component A must be computed before component B.
Furthermore, these components must always be computed the same number of times (first component A, then component B, each time).
Asynchronous coagent networks do not share these limitations.

\subsection{Non-Differentiable Network Components}
\label{sec:non-differentiable}

Another major limitation of BDL is the requirement that the network be differentiable and deterministic.
While some non-deterministic components can be incorporated using the well-known reparameterization trick, this trick does not generalize to many types of stochastic components.
For example, the binary nodes that we use in this paper's experiments cannot be handled by the reparameterization trick, and so cannot be built into networks designed with the standard BDL approach.

Another example is the memory cells used in architectures such as neural Turing machines \citep{graves2014neural} and differentiable neural computers \citep{graves2016hybrid}.
The read-write memory cells must be carefully designed around the differentiability requirement.
The coagent framework could broaden the design space of architectures like these and allow researchers to explore new approaches, less restricted by the requirement that every single network component be differentiable.

\subsection{Hierarchical Networks and Exploration}
\label{sec:hierarchical}

Note that this use is closely related to, but distinct from, the use discussed in Section \ref{sec:temporalAbstraction}.

Coagent networks (asynchronous or not) allow for exploration at higher levels than the output layer: since networks have internal nodes that are themselves stochastic RL algorithms, internal components of the network that do not directly affect the action can explore, allowing for hierarchical architectures that facilitate exploration.

Most RL algorithms are limited by ability to explore only at the level of primitive actions (that is, at the level of their action space). 
That is, in order to learn via trial and error, they sometimes select actions they expect are suboptimal in order to gain additional information about whether, perhaps, those other actions are actually optimal.
Although there are many strategies for determining how and when to explore, these strategies almost all focus on exploring at the level of primitive actions (changing the action chosen at the level of the action space).

To see why this is a significant limitation, consider what this would mean if one's brain were to be replaced with any standard RL algorithm and one was tasked with learning to play chess (including learning to sit up in a chair, move the pieces, etc.).
This style of exploration equates to randomly stimulating different muscles--twitching one's finger or eye slightly for one time step.
Clearly learning to play chess requires exploration at a higher level--at the level of selecting different moves in the game.
Exploration at the level of primitive actions has thus far been successful because researchers carefully design the problems RL agents face so that exploration at the level of primitive actions is appropriate for the task at hand.
However, the inability of standard RL agents to learn to explore at higher levels limits them to relatively shallow capabilities, even when using deep networks, since these deep networks are held back by exploration only occurring at their outputs.

RL researchers have developed a variety of hierarchical reinforcement learning techniques, like the options framework, to enable agents to explore at higher levels--often with temporal abstraction (this allows for exploration to result in a completely different sequence of actions for an extended period of time, like selecting a different move in chess rather than twitching a finger differently for one time step).
However, hierarchical RL methods tend to be relatively heuristic, searching for temporally extended actions (called skills) that achieve heuristic goals like reaching distant parts of state space, or reaching bottleneck states.

Coagent networks present an exciting alternative to these approaches, by allowing for exploration at \emph{any} level within a parameterized policy.
For example, they can represent and train neural networks where every single node is stochastic, or hand-crafted network architectures that are equivalent to the options framework (see Section \ref{sec:temporalAbstraction}).

\section{Combining Coagent Learning Rules with Backpropagation}
\label{sec:combining_backprop}

While the coagent framework can be viewed as an alternative to backpropagation-based deep learning (BDL), the two approaches can also be combined to leverage the advantages of both.
Specifically, coagent learning rules can be combined with backpropagation by designing coagents which are themselves deep networks.
For these coagents, the gradient signal from the coagent learning rule is then propagated internally using backpropagation.

\begin{figure}[h]
    \includegraphics[width=1.0\linewidth]{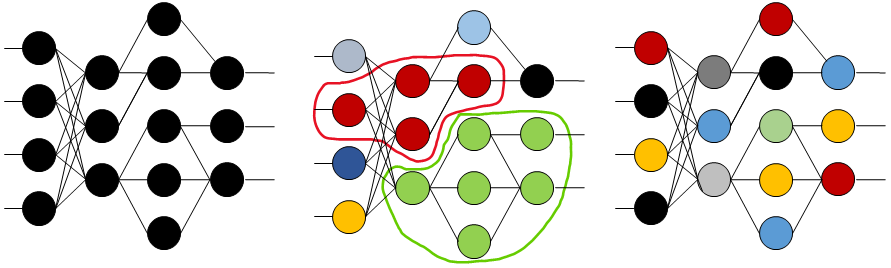}
    \centering
\caption{An illustrative example to show how coagent learning rules may be combined with backpropagation.
}
\label{fig:coagents_and_backprop}
\end{figure}

For an illustrative example, consider Figure \ref{fig:coagents_and_backprop}.
From left to right, these figures represent a standard (backpropagation-based) deep network, a coagent network that combines backpropagation with coagent learning rules, and a standard coagent network that does not use backpropagation.
In the left and middle networks, groups of connected nodes of the same color represent a deterministic deep network that can be trained with backpropagation. 
For the middle network, the two BDL networks are highlighted with a colored line.
In the middle and right networks, all other nodes represent coagents that are linear function approximators.

The cases of the left and right networks are straightforward.
The 15 linear coagents in the right network are simply trained with a standard coagent learning rule (based on theory such as Theorem \ref{thm:spat}).
The BDL network on the left is trained using the standard DRL approach: the learning signal is backpropagated through the network.

The middle network combines the two cases.
This network consists of seven coagents, two of which are deep networks (the red and green groups).
During training, for those two coagents, the respective gradient from the coagent learning rule is simply backpropagated through each coagent.

Intuition for this process can be built by considering a standard BDL network, such as the one on the left.
Suppose the network on the left is a DRL policy network trained using a policy-gradient (PG) algorithm (for example, PPO).
The PG algorithm provides a gradient, and the network backpropagates that gradient during an update.
The process for BDL coagents is identical:
The coagent learning rule provides a gradient, and the coagent backpropagates that gradient during an update.

In this way, practitioners can design coagents that are themselves deep networks, and potentially leverage the advantages of both DRL and coagent networks for a single network.
%

\section{Experiment Details}
\label{sec:experiment_details}

\subsection{Coagent Network Details}

The coagent network's learning algorithm is simply the ``Actor-Critic with Eligibility Traces (episodic)'' algorithm described \citet{sutton2018reinforcement}, without the $I$ component (which is a term in the actor's gradient calculation in the algorithm described in the text).
We normalize the coagent network's observations (similarly to what most DRL implementations do, including the PPO implementation we use).

The actor network is a two-layer network comprised of $40$ coagents.
Each coagent is a linear function approximator (the identity basis with a bias term) with softmax action selection.
The first layer is comprised of $32$ 224-parameter coagents, each of which outputs a binary signal to the second layer.
These binary signals are normalized to the range $[-1, 1]$, so each second-layer coagent sees $32$ inputs in \{-1, 1\}.
The second layer is comprised of eight 363-parameter coagents corresponding to the eight-dimensional action space of the environment.
Each action dimension is discretized into 11 possible actions for each of the eight second-layer coagents: $[-1, -0.32, -0.1, -0.032, -0.01, 0.0, 0.01, 0.032, 0.1, 0.32, 1]$.

The critic network is a BDL fully-connected network.
It has two hidden layers with 64 units each, and uses the tanh activation function.

Both the actor coagents and the critic use the Adam optimizer \citep{kingma2014adam}.

Hyperparameters used follow:
actor learning rate = 0.003844152898051233,
batch size actor = 128,
Adam beta1 for actor = 0.9,
Adam beta2 for actor = 0.9914815877866358,
$\lambda$ actor = 0.8206776332879618,
batch size critic = 32,
critic learning rate = 0.00018759145359217475,
Adam beta1 for critic = 0.0,
Adam beta2 for critic = 0.9999975597793732,
$\lambda$ critic = 0.0,
$\gamma$ = 0.9679813850692468,
Adam epsilon (actor and critic) = 0.00000001.

\subsection{SOTA DRL Baseline Details}
\label{sec:SOTA_details}

For the state-of-the-art (SOTA) returns that we provide for comparison, we trained $30$ BDL agents trained with proximal policy optimization (PPO).
We used an implementation based on the ppo\_continuous\_action implementation of the Clean-RL library \citep{cleanrl}.

We used the following hyperparameters (from the work of \citet{rl-zoo3}): learning rate = 1.90609e-05, gamma = $0.98$, gae\_lambda = $0.8$, update\_epochs = $10$ , norm\_adv = False, clip\_coef = 0.1, clip\_vloss = False, ent\_coef = 4.9646e-07, vf\_coef = 0.677239, max\_grad\_norm = 0.6, target\_kl = None, minibatch\_size = 32, adam\_eps = 1e-5, num\_envs = $6$, num\_steps = $512$.

\subsection{Other Experimental Details}

For Table \ref{table:results}, the reported standard deviations are calculated by finding the 1000 standard deviations (corresponding to the 1000 episodes) across the (30 or 100) trials, and then reporting the mean of those 1000 standard deviations.

\end{document}